\documentclass{article}

\usepackage[preprint]{neurips_2026}

\usepackage[utf8]{inputenc} 
\usepackage[T1]{fontenc}    
\usepackage{hyperref}       
\usepackage{url}            
\usepackage{booktabs}       
\usepackage{amsfonts}       
\usepackage{nicefrac}       
\usepackage{microtype}      
\usepackage{xcolor}    
\usepackage{amsmath, amsfonts, amssymb, amsthm}     
\usepackage{enumitem}
\usepackage{algorithm}
\usepackage{algpseudocode}
\usepackage{hyperref}
\usepackage{graphicx}
\usepackage{float}
\usepackage{url}
\usepackage{mathrsfs}
\usepackage{bm}

\allowdisplaybreaks
\title{Learning stochastic multiscale models through normalizing flows
\thanks{This research was supported in part by NSF grant DMS-2246815 and by the Simons Foundation through the Travel Support for Mathematicians program.}
}

\author{%
	Anan Saha \\
	Department of Mathematics\\
	Louisiana State University\\
	Baton Rouge, LA 70802 \\
	\texttt{asaha9@lsu.edu} \\
	\And
	Arnab Ganguly \thanks{Corresponding author.}\\
	Department of Mathematics\\
	Louisiana State University\\
	Baton Rouge, LA 70802 \\
	\texttt{aganguly@lsu.edu} \\
}

\newcommand{\Rt}{\longrightarrow}
\newcommand{\rt}{\rightarrow}

\newcommand{\lRT}{\Longrightarrow}

\newcommand{\hs}{\hspace}
\newcommand{\vs}{\vspace}

\newcommand{\mrm}{\mathrm}
\newcommand\SC{\mathcal}
\newcommand{\scr}{\mathscr}

\newcommand{\mfk}{\mathfrak}

\newcommand{\ep}{\epsilon}
\newcommand{\vep}{\varepsilon}

\newcommand{\s}{\sigma}

\def\vart{\vartheta}
\def\l{\lambda}

\newcommand{\f}{\frac}
\def\ot{\otimes}

\def\<{\langle}
\def\>{\rangle}
\def\~{\tilde}
\newcommand{\dfeq}{\stackrel{def}=}
\newcommand{\eqd}{\stackrel{\mrm{Law}}=}
\def \triple|{|\! | \! |}

\newcommand*{\argmax}{\mrm{arg\,max}}
\newcommand*{\argmin}{\mrm{arg\,min}}

\renewcommand{\leq}{\ensuremath{\leqslant}}
\def\lf{\left}
\renewcommand\le{\ensuremath{\leqslant}}
\def\ri{\right}

\renewcommand\geq{\ensuremath{\geqslant}}
\renewcommand\ge{\ensuremath{\geqslant}}

\newcommand{\EE}{\mathbb{E}}
\newcommand{\PP}{\mathbb{P}}

\def\R{\mathbb R}

\newcommand{\bx}{\bm{x}}

\newcommand{\drft}{b}

\newcommand{\diffus}{\sigma}
\newcommand{\dffun}{\sigma}

\newcommand{\rfdist}{\nu_{\mrm{ref}}}

\newcommand{\inv}{\pi}

\newcommand{\post}{p_{\mrm{post}}}

\newcommand{\loss}{\mathcal{L}}

\newcommand{\indic}{\mathbf 1}

\newtheorem{theorem}{Theorem}[section]
\newtheorem{lemma}[theorem]{Lemma}
\newtheorem{assumption}[theorem]{Assumption}

\begin{document}
\maketitle

	\begin{abstract}
		
		Many systems in physics, engineering, and biology exhibit multiscale stochastic dynamics, where low-dimensional slow variables evolve under the influence of high-dimensional fast processes. In practice, observations are often limited to a single trajectory of the slow component, while the fast dynamics remain unobserved, making statistical learning challenging. 
		Approaches based on partial differential equations (PDE), such as Fokker--Planck formulations, aim to characterize the evolution of probability densities, typically requiring dense space–time data or grid-based solvers. In contrast, we adopt a trajectory-based perspective and develop a data-driven framework for learning effective stochastic dynamics from a single observed path. 
		We model the dynamics by coupled multiscale stochastic differential equations (SDEs) and  first obtain a principled model reduction through stochastic averaging. Unlike generic model reduction techniques like PCA, this respects the dynamical structure of the original system and explicitly incorporates the interaction between slow and fast scales. 
		A central challenge, however, is that the reduced model depends on the invariant distribution of the fast process which is a solution to an intractable and often unknown PDE. We introduce a novel learning framework that parameterizes the invariant distribution using normalizing flows, enabling expressive density modeling in the latent fast-variable space. The flow is trained end-to-end by optimizing a penalized likelihood objective induced by the reduced stochastic dynamics.
		Furthermore, we develop a Bayesian variational inference procedure for uncertainty quantification, employing a second normalizing flow to approximate the posterior distribution over model parameters. This yields a scalable approach to capturing epistemic uncertainty in multiscale systems.
		Our method bridges stochastic model reduction and flow-based latent density modeling, offering a scalable approach to learning effective dynamics in partially observed multiscale stochastic dynamical systems.

	\end{abstract}

\section{Introduction} \vs*{-0.1cm}
Multiscale behavior is a defining feature of many stochastic systems arising in physics, engineering, and biology, where a few slow variables interact with rapidly evolving high-dimensional components \cite{PavStu08}. Typical examples include molecular systems on complex energy landscapes, stochastic reaction networks with widely separated reaction rates, and climate systems coupling large-scale atmospheric evolution with rapidly fluctuating turbulence. In practice, one often has access only to a single observed trajectory of the slow variables, while the fast variables remain unobserved.

Mathematically, such systems can be described through PDE formulations, such as Fokker--Planck or Kolmogorov equations, which govern the evolution of probability densities on the full state space \cite{FSR11, PiMaBe13, SpSuZh25}. These formulations are theoretically complete and useful in settings where either access to evolving densities or fine discretizations over the state space is available. However, their numerical implementation can become intractable in high-dimensional settings due to the cost  solving PDEs. Furthermore, these approaches are not well-suited to  data, consisting of a single observed trajectory.

A complementary viewpoint is provided by multiscale stochastic differential equations (SDEs) with scale separation. In this regime, a principle model-reduction can be achieved through stochastic averaging. This leads to a reduced-order SDE for the slow variables alone driven by averaged drift and diffusion functions where the effect of the fast variables are averaged out with repect to its invariant distribution. This reduction eliminates explicit dependence on fast variables while retaining the underlying dynamical structure.

The resulting reduced model, however, depends critically on the invariant distribution of the fast dynamics, which is typically unknown and only characterized implicitly through a high-dimensional stationary PDE. Consequently, even when the full multiscale system is specified, the effective coefficients governing the slow dynamics are not directly computable. This gives rise to a fundamental inference problem: learning the averaged driving functions, and hence the effective SDE dynamics, from partial trajectory observations in the presence of latent fast-scale equilibrium effects.

A natural statistical formulation of this problem is through maximum likelihood estimation (MLE), where the reduced SDE likelihood is optimized over all admissible probability measures on the state space of the fast variables (subject to suitable moment conditions). However, this infinite-dimensional optimization problem is computationally intractable and must be replaced by a tractable surrogate one that admits efficient parameterization and enables efficient optimization.

\textbf{Contributions:} The main contributions of this work are:

\begin{itemize}
    \item We formulate inference for multiscale SDEs as a latent-invariant measure learning problem, where the effective drift and diffusion are expressed as conditional expectations with respect to an unknown invariant distribution of the fast process.

    \item We introduce a normalizing-flow parameterization of the latent invariant measure, which facilitates accurate approximation of the averaged  functions governing the reduced dynamics.

    \item We develop a Bayesian extension using a second normalizing flow to approximate the posterior distribution over flow parameters, enabling uncertainty quantification for the inferred effective dynamics.

    \item  We show that, under suitable  conditions, the proposed flow-based estimators converge to MLE  as the complexity of the underlying neural network (NN) increases.

    \item We validate the proposed method on synthetic multiscale systems, demonstrating accurate recovery of effective drift, matching of the true and estimated slow trajectories, as well as meaningful uncertainty estimates.
\end{itemize}

\vs{-0.1cm}
Overall, our approach bridges stochastic model reduction and flow-based probabilistic inference, providing a scalable framework for learning effective dynamics in partially observed multiscale stochastic systems.
	
	\section{Multiscale systems: Problem Setup} \vs*{-0.1cm}
	
	We consider multiscale stochastic system, where an observed slow component $X^{(n)}$ is coupled with a rapidly evolving latent process $Y^{(n)}$. 
	The slow process \(X^{(n)}\) satisfies the It\^o SDE
	{\setlength{\abovedisplayskip}{6pt}
		\setlength{\belowdisplayskip}{6pt}
		\begin{align}\label{eq:slow_sde}
			\setlength{\abovedisplayskip}{6pt}
			\setlength{\belowdisplayskip}{6pt}
			dX^{(n)}(t)
			=
			\drft\bigl(X^{(n)}(t), Y^{(n)}(t)\bigr)\,dt
			+
			\dffun\bigl(X^{(n)}(t), Y^{(n)}(t)\bigr)\,dW(t),
	\end{align}}
	where \(W\) is an \(\R^d\)-valued Brownian motion. The latent process \(Y^{(n)}\) evolves on the fast time scale \(O(1/n)\). We assume that the state spaces of \(X^{(n)}\) and \(Y^{(n)}\) are \(\R^d\) and \( \R^{d'}\), respectively. 
	\(
	\drft: \R^d\times \R^{d'}\to \R^d
	\), 	
	\(
	\dffun:\R^d\times \R^{d'}\to \R^{d\times d}
	\)
	denote the drift and diffusion coefficients and the parameter $n$ captures the differences in speeds of the components.
	
	To formalize the multiscale structure, we assume that $Y^{(n)}$ and the joint process \((X^{(n)},Y^{(n)})\) are Markov with infinitesimal generator \(\mathcal{A}_n\) of  \((X^{(n)},Y^{(n)})\)  given by
	{\setlength{\abovedisplayskip}{6pt}
		\setlength{\belowdisplayskip}{6pt}
		\begin{align*}
			\mathcal{A}_n f(x, y) = \mathcal{A}^0_y f(\cdot, y)(x) + n \mathcal{A}^1f(x, \cdot)(y),
		\end{align*}
	}
	where
	{\setlength{\abovedisplayskip}{2pt}
		\setlength{\belowdisplayskip}{2pt}\begin{align*}
			\mathcal{A}^0_y \phi(x) = b(x, y) \cdot \nabla_x \phi(x) + \mathrm{trace}\left( \sigma \sigma^\top(x,y)\nabla_x^2 \phi(x) \right).
	\end{align*}}
	Here \(\mathcal{A}^0_y\) describes the evolution of the slow variable when the fast component is  at state \(y\), whereas \(n\mathcal{A}^1\) is  the generator of $Y^{(n)}$. If $Y^{(n)}$ is another It\^ o SDE of the form 
	{\setlength{\abovedisplayskip}{6pt}
		\setlength{\belowdisplayskip}{6pt}
		\begin{align}\label{eq:fast_sde}
			\setlength{\abovedisplayskip}{6pt}
			\setlength{\belowdisplayskip}{6pt}
			dY^{(n)}(t)
			=
			n\beta\bigl(Y^{(n)}(t)\bigr)\,dt
			+
			\sqrt n \alpha\bigl(Y^{(n)}(t)\bigr)\,d\tilde W(t),
	\end{align}}	
	then \(\mathcal{A}^1\) is given by
	\begin{align}\label{eq:gen-Y-SDE}
		\mathcal{A}^1\psi(y) = \beta(y)\cdot \nabla_y\psi(y)+\mrm{trace}(\alpha\alpha^\top(y)\nabla^2_y\psi(y)).
	\end{align} 
	If $Y^{(n)}$ is a continuous time Markov chain (CTMC) with state space $\R^{d'}$ then \(\mathcal{A}^1\) has the form 
	\begin{align}\label{eq:gen-Y-jump}
		\mathcal{A}^1\psi(y) = c(y)\int (\psi(y') - \psi(y)) \nu(y, dy'),
	\end{align} 
	where $c$ is the jump intensity and $\nu$ is the (post-jump) transition probability kernel.	
	
	\noindent
	{\bf Goal:} We consider the scenario when the fast $Y$-dynamics is unknown or intractable and  the goal is to infer the effective dynamics of the system from observations of the slow $X$-component only.  Specifically, we want to find a data-driven effective drift function $b_{\mrm{eff}}: \R^d \rt \R^d$ and the diffusion function $\sigma_{\mrm{eff}} : \R^d \rt \R^{d\times d}$, such that the SDE $X_{\mrm{eff}}$ driven by  them,
	\begin{align*}
		dX_{\mrm{eff}}(t) = b_{\mrm{eff}}(X_{\mrm{eff}}(t)) dt + \sigma_{\mrm{eff}}(X_{\mrm{eff}}(t)) dW(t)
	\end{align*}
	mimics the dynamics of $X^{(n)}$.
	
	\looseness=-1
	This learning problem is intrinsically challenging because the fast $Y$-component  is unobserved and typically has unknown dynamics. If the fast dynamics are known or can be reliably modeled, a naive inference strategy would attempt to reconstruct the latent trajectory of \(Y^{(n)}\) jointly with the slow process. But this strategy is impractical and computationally expensive as $Y^{(n)}$ moves rapidly with much shortertime scale $O(1/n)$. For instance, if \(Y^{(n)}\) is a fast continuous-time Markov chain with generator  \(n \mathcal{A}^1\) as in \eqref{eq:gen-Y-jump}, then over any fixed time horizon \([0,T]\) one must precisely track \(O(n)\) jump events.
	As a result, classical inference methods for stochastic differential equations (e.g., \cite{Yos92, Kes97, Chib01, RoSt01, Saha08, GoWi08, Iacus08, ChCh11, ArOp11, CsOp13, Li13, BlSo14, WGBS17, Bis22a}) are not directly applicable in this partially observed multiscale setting.

	
	A naive inference strategy would attempt to reconstruct the latent trajectory of \(Y^{(n)}\) jointly with the slow process. This is  infeasible as the dynamics of $Y^{(n)}$ is typically not known. But even when it is known, the strategy is impractical and computationally expensive as $Y^{(n)}$ moves rapidly with time scale $O(1/n)$. For instance, if \(Y^{(n)}\) is a fast continuous-time Markov chain with generator  \(n \mathcal{A}^1\), then that means over any fixed time horizon \([0,T]\) one needs to precisely track \(O(n)\) jumps.

	\section{Methodology}	
	\textbf{Model Reduction through Stochastic averaging: }
	Our first step is to obtain a model reduction of the given stochastic system. Standard techniques such as PCA and its variants are not suitable in this setting, as they do not respect the underlying dynamics. Instead, the reduction must be carried out in a principled manner that preserves the evolution of the slow $X$-component. This can be achieved via the stochastic averaging principle, and the following result, whose proof is given in the appendix, serves as the starting point for our inference problem.

	\begin{assumption}\label{assum:sto-avg}
		There exists $V \in C^2(\R^d, [1,\infty))$ (called a Lyapunov function) such that the following hold:
		\begin{enumerate}[label=(\roman*), ref=(\roman*)] 
			\item \label{cond:item:lyapunov-pull} $c_* \dfeq -\limsup_{\|y\| \rt \infty}\mathcal{A}^1 V(y)>0$ 
			\item \label{cond:item:growth-diff} $\nabla^\top V(y)(\alpha\alpha^\top(y))\nabla V(y) = o(V(y)|\mathcal{A}^1 V(y)|)$ as $\|y\| \rt \infty$, that is, 
			$$ \f{\nabla^\top V(y)(\alpha\alpha^\top(y))\nabla V(y)}{V(y)|\mathcal{A}^1V(y)|} \ \stackrel{\|y\| \rt \infty}\Rt 0;$$
			\item for some constant $C_{0}$, exponents $p_0 \geq 0, 0\leq \alpha_b, \alpha_{\sigma} \leq 1$,
			\begin{align*}
				\max\lf\{\|b(x,y) - b(x,y')\|, \|\sigma(x,y) - \sigma(x,y')\|\ri\}  \leq C_0 V(y)^{p_0} \|x - x'\|^{\alpha_0}.
			\end{align*}
			\item (uniform ellipticity) for some constant $\lambda>0$, $u^\top \alpha\alpha^\top(y) u \geq \lambda \|u\|^2 $ for any $y, u \in R^d.$
		\end{enumerate}
	\end{assumption}

	\begin{theorem} \label{th:sto-avg} Suppose that $Y^{(n)}$ satisfies the SDE \eqref{eq:fast_sde} and  Assumption \ref{assum:sto-avg} hold.  Then, 	as \(n \to \infty\), the slow process \(X^{(n)} \stackrel{n \rt \infty} \lRT X_{\inv_0}\), where 
		{\setlength{\abovedisplayskip}{6pt}
			\setlength{\belowdisplayskip}{4pt}
			\begin{align}\label{eq:int-drift-sde}
				\begin{aligned}
					&\mrm{d}X_{\inv_0}(t) = \bar\drft_{\inv_0}(X_{\inv_0}(t))\,\mathrm{d}t + \bar{\sigma}_{\inv_0}(X_{\inv_0}(t))\,\mathrm{d}W(t), \\
					\bar\drft_{\inv_0}(x) &= \int_{\R^{d'}} b(x, y)\, \inv_0(x,\mathrm{d}y), \quad
					\bar{\sigma}_{\inv_0}(x) = \bigg(\int_{\R^{d'}} \s\s^\top(x,y) \inv_0\mathrm{d}y)\bigg)^{1/2},
				\end{aligned}
			\end{align}
			and $\inv_0$ is the unique stationary distribution of 	 \(\mathcal{A}^1\) in \eqref{eq:gen-Y-SDE}	, i.e., 
			{\setlength{\abovedisplayskip}{6pt}
				\setlength{\belowdisplayskip}{4pt}
				\begin{align}\label{eq:inv-dist-pde}
					(\mathcal{A}^1)^*\inv_0 = 0 \quad   \Longleftrightarrow 	\quad	\int_{\R^{d'}} \mathcal{A}^1_x \psi(y)\,\inv_0(dy) = 0
					\quad \text{for all test functions } \psi.
				\end{align}	
			}
		}
		
	\end{theorem}
	The above result is also true when $Y^{(n)}$ is a jump-Markov process having generator $n\mathcal{A}^1$ with $\mathcal{A}^1$ as in \eqref{eq:gen-Y-jump}. This provides a mathematical justification for a reduced model in which the influence of the high-dimensional fast variables enters only through their equilibrium effect on the slow dynamics, leading to a significant reduction in complexity. In other words, the limiting  process $X_{\inv_0}$ is playing the role of $X_{\mrm{eff}}$ mentioned earlier with $b_{\mrm{eff}} \equiv \bar b_{\inv_0}$ and $\sigma_{\mrm{eff}} = \bar \sigma_{\inv_0}$, and the original inference problem is reformulated as one of  learning $\bar b_{\inv_0}$ and $\bar \sigma_{\inv_0}$ from the available $X$-observations.

	However the main difficulty now lies in the dependence of   $\bar b_{\inv_0}$ and $\bar \sigma_{\inv_0}$  on the stationary distribution $\inv_0$. This probability measure is unknown when the fast dynamics are not specified, but even when they are, $\inv_0$ cannot be computed efficiently as it is characterized implicitly through an intractable high-dimensional PDE as in \eqref{eq:inv-dist-pde}.
	
	Existing works on multiscale systems \cite{PavStu07, PaPoStu12, GaSp18, AbGa23}  considered settings where the reduced drift has a known parametric structure, for example,
	\(
	\bar\drft_{\inv_0}(x) = \vart_0 \tilde b(x),
	\)
	where \(\tilde b\) is known and \(\vart_0 \equiv \vart_{\inv_0}\) is a finite-dimensional parameter depending on $\inv_0$  (e.g. moments of $\inv_0$). In such cases, the inference task simply reduces to estimating the finite-dimensional parameter \(\vart_0\) by standard methods.
	
	In contrast, we consider a substantially more challenging setting in which the structural form of the averaged drift \(\bar\drft_{\inv_0}\) is completely unknown. Our aim is therefore to learn the effective dynamics in a fully nonparametric fashion, combining model reduction ideas from stochastic averaging with modern statistical and machine learning methodology.

	\textbf{A note about the learning problem: } Crucially, our inferential target is \(\bar\drft_{\inv_0}\), not \(\inv_0\). In general, the mapping \(\inv \in \mathcal{P}(\R^d) \mapsto \bar\drft_\inv(\cdot)\) is not injective. For example, if \(b(x,y)=b_0(x)y\), then \(\bar\drft_\inv(x)=b_0(x)\mfk{m}\) depends only on the mean \(\mfk{m}\) of \(\inv\), so infinitely many distinct invariant measures produce the same effective drift. For the class of SDEs considered here, it is \(\bar\drft_{\inv_0}\) that determines the reduced dynamics, governs the macroscopic behavior of the system, and is directly linked to the observable data through \(X\). It is therefore the primary statistical object of interest: unlike \(\inv_0\), it is typically identifiable at the level of the reduced model and practically estimable from $X$-data. Nevertheless, if \(b(\cdot,\cdot)\) is a characteristic kernel so that the map \(\pi \mapsto \bar\drft_\inv(\cdot)\) becomes injective, then our method also recovers true \(\inv_0\).


	\subsection{Inference Setup}

	For simplicity, we assume in this paper that the diffusion function $\dffun$ is known and does not depend on $Y$-component, i.e., $\sigma(x,y) = \sigma(x)$. Thus the inferential object is only  \(\bar\drft_{\pi_0}(\cdot) = \int b(\cdot,y)\,\inv_0(dy)\). We assume the kernel $b(\cdot,\cdot)$ is specified, so all uncertainty in the effective drift \(\bar\drft_{\pi_0}(x) = \int b(x,y)\,\pi_0(dy)\) arises from the latent distribution \(\pi_0\), rather than from the local interaction structure \(b(x,y)\) itself. This estimation task is particularly challenging because the unknown probability measure \(\pi_0\) is an infinite-dimensional 
	\(\mathcal{P}(\R^{d'})\)-valued parameter, which cannot be estimated by conventional techniques. 
	
	{\em Data and inference task: } Our data consists of a high-frequency observation vector  $\bx_{0:M_0} \dfeq (x_0,x_1, \hdots,x_{M_0})$ which is a realization of $(\bar X_{\inv_{0}})_{t_0 :t_{M_0}} \dfeq  (\bar X_{\inv_{0}}(t_0), \bar X_{\inv_{0}}(t_1), \hdots, \bar X_{\inv_{0}}(t_{M_0}))$ --- discrete observations from $\bar X_{\inv_0}$ at time points $\{t_m:m =01,2,\hdots, M_0\}$  satisfying $\Delta = t_m - t_{m-1} \ll 1$.	 Our objective is to compute an estimator $\hat{\bar b}_{\inv_0}$ based on the data $\bx_{0:M_0}$.

	We assume that  the unknown $\pi_0 \in \SC{P}^{(p)}_{\text{abs}}(\R^{d'})$ for some $p \geq 0$.  By Euler-Maruyama approximation, the (random) likelihood functional $L_T: C(\scr{X}, \R^d) \to \R$ for estimating the drift function $\bar\drft_{\pi_0}$ is given by
	\begin{align}
		\label{eq:abs-like}
		\begin{aligned}
			&L_T(\beta \mid   \bx_{0:M_0}) =\ \prod_{m=0}^{M_0-1} \f{1}{(2\Delta\pi)^{d/2} \mrm{det}^{1/2}(A(x_m))} e^{-R(\bx_{0:M_0)}) /2\Delta}\\
			\propto&\ \exp \lf\{ \sum_{m=0}^{M_0-1}  \beta^\top(x_m)A^{-1}(x_{m}) (x_{m+1}-x_{m-1})  - \f{\Delta}{2}\sum_{m=0}^{M_0-1}  \beta^\top(x_m)A^{-1}(x_{m})\beta(x_m) \Delta\ri\},
		\end{aligned}	
	\end{align}
	where $R(\bx_{0:M_0)}) \dfeq (x_{m+1}-x_m - \beta(x_m)\Delta)^\top A^{-1}(x_m)(x_{m+1}-x_m - \beta(x_m)\Delta)$ and $A(x)=\dffun\dffun^\top(x)$. An ideal estimator  of course belongs to the class of  maximum likelihood estimators (MLEs) defined as
	\begin{align}
		\label{eq:MLE-class-def}
		\mathcal{M}_{\mrm{mle}} =\lf\{\bar \drft_{\hat\rho_{*}}(\cdot)\equiv \int b(\cdot,y) \hat \rho_{*}(dy) : \hat \rho_{*} \in \mathcal{A} \ri\},
	\end{align}
	where 
	\begin{align}\label{eq:agmax-like}
		\hat \rho_{*} \in \mathcal{A}  \dfeq \argmin \lf\{-\ln L_T( \bar\drft_\rho \mid \bx_{0:M_0}): \rho \in \SC{P}^{(p)}_{\text{abs}}(\R^{d'})\ri\}.
	\end{align}
	Notice that $\mathcal{A}$ and $\mathcal{M}_{\mrm{mle}}$  might not be singleton, i.e., a MLE estimator of $\bar\drft_{\pi_0}$ is not unique.

	\subsection{Estimation procedure}
	Clearly, the optimization problem above is computationally infeasible, unless the state space of $Y$-component is simple (e.g., finite) or $b(\cdot,\cdot)$ has a simple structure (e.g., $\drft(x,y)$ is separable in $x$ and $y$ in the sense $\drft(x,y) = \drft_0(x)g_0(y)$). We approximate $\SC{P}^{(p)}_{\text{abs}}(\R^{d'})$ by a flexible parametric family $\SC{P}_n(\R^{d'})$ via  flow $f_\theta$: 
	\[
	\SC{P}_n(\R^{d'}) = \Big\{q^{(n)} \in \SC{P}(\R^{d'}): q^{(n)} \equiv q^{(n)}_{\theta_n}  = (f^{(n)}_{\theta_n})_{\#} \rfdist, \ \theta_n \in \Theta_n \subset \R^{d'}\Big\},
	\]
	where $\#$ denotes the push-forward, and $f^{(n)}_{\theta_n}$ is modeled as a neural network (NN) with increasing complexity (depth, width) captured by $n$. Clearly, $Z \sim \rfdist \implies f_\theta(Z) \sim q_\theta(\cdot)$.  Unlike parametric families such as Gaussian mixtures, {\em flows}
	represent complex probability laws via a transformation of a 
	simple reference distribution \cite{ReSh15, PaNaRe21}. 
	Mathematically, the justification of replacing $\SC{P}^{(p)}_{\text{abs}}(\R^{d'})$ comes from Lemma \ref{lem:wass-density-flows}, which shows any probability distribution in 	$\SC{P}^{(p)}_{\text{abs}}(\R^{d'})$ can be approximated by $\SC{P}_n(\R^{d'}) $ as the complexity of the underlying neural network increases.

\begin{lemma}[Wasserstein convergence of neural pushforwards]
		\label{lem:wass-density-flows}
		Assume that for some $p \ge 1$, 
		\(
		\nu_{\mathrm{ref}} \in \mathcal P^{(p)}_{\mathrm{abs}}(\R^{d'})
		\).
		Fix any 
		\(
		\pi_0 \in \mathcal P^{(p)}_{\mathrm{abs}}(\R^{d'})
		\). Then there exists a sequence of neural networks $\{f_{\theta_n}: \theta_n\}$ with non-polynomial activation function such that the family $\SC{P}_n(\R^{d'})$ is dense in $\SC{P}^{(p)}_{\mathrm{abs}}(\R^{d'})$ under the Wasserstein-$p$ metric, $\mathcal{W}_{p}$; specifically, for any $\inv \in \SC{P}^{(p)}_{\text{abs}}(\R^{d'})$
		\begin{align*}
			\lim_{n \rt \infty} \inf_{\theta_n \in \Theta_n}\mathcal{W}_p\lf(q^{(n)}_{\theta_n}, \inv\ri) =0.
		\end{align*}

		
	\end{lemma}
	\looseness=-1
The proof of the lemma is based on the Universal Approximation Theorem (UAT) and is provided in the appendix for completeness. Under certain extra assumptions, approximation rates of order $O(n^{-1/d})$ in $\mathcal{W}_p$-metric can be obtained, where $n$ denotes the number of network parameters \cite{YaLiWa22}.
	Thus flows form a highly expressive class of probability distributions suitable for \emph{likelihood-based} inference problem that can approximate  non-Gaussian, multimodal, and heavy-tailed distributions while still admitting 	efficient gradients and scalable optimization (for suitable architectures).
	
	{\em Flow based estimator:}  We now define  the flow based estimator of $\bar b_{\pi_0}$. The main idea is to restrict the admissible class of probability measures in \eqref{eq:agmax-like} to the family of pushforward measures $\mathcal{P}_{n}(\R^{d'}) $ but to ensure stability of the optimizer we look at a regularized (or penalized) optimization problem. For a probability measure $\rho \in \SC{P}^{(p)}_{\text{abs}}(\R^{d'})$, denote its $p$-th moment by $\mfk{m}_p(\rho) = \int |y|^p\,\rho(dy)$, and define the penalized loss function $\loss_\lambda: \SC{P}^{(p)}_{\text{abs}}(\R^{d'}) \rt \R$ as 
	\begin{align*}
		\loss_\lambda(\rho) = -\ln L_T( \bar\drft_\rho \mid \bx_{0:M_0})+\lambda \mfk{m}_p(\rho)
	\end{align*}
	Next let
	\begin{align}\label{eq:NN-para-min}
		\hat \theta_{n,\lambda} = \text{argmin}_{\theta_n \in \Theta_n} \loss_\lambda(q^{(n)}_{\theta_n}).
	\end{align}
	and define the flow-based estimator of $\bar b_{\pi_0}$ as
	\begin{align}
		\label{eq:theta-est}
		\bar b_{\hat \theta_{n,\lambda}}(\cdot)  \dfeq \bar b_{q^{(n)}_{\hat \theta_{n,\lambda}}}(\cdot) \equiv \int_{\R^{d'}}b(\cdot,y)q^{(n)}_{\hat \theta_{n,\lambda}}(dy)  \equiv \int_{\R^p}b(\cdot,f_{\hat \theta_{n,\lambda}}(z))\rfdist(dz) 
.	\end{align}

	%
	
	\begin{theorem}[Penalized parametric minimizers approach the global MLE set]		
		\label{thm:pen-mle-conv}
		Fix $p > 1$ and define the MLE-set $\mathcal{M}_{\mrm{mle}}$ and $\mathcal{A}$ by \eqref{eq:MLE-class-def} and \eqref{eq:agmax-like}.
		Suppose the drift-kernel $b(\cdot,\cdot)$ satisfies the following Lipschitz condition: there exists a locally bounded function  $C_b: \R^d \rt \R^d$ and exponent $0\leq q_0 <p-1$ such that
		\begin{align}
			\label{eq:lip-b-cond}
			\ \|b(x,y)-b(x,y')\| \le C_b(x)(1+\|y\|^{q_0}+\|y'\|^{q_0})\|y-y'\|.
		\end{align}
		Then for every $1\le p'<p$,
		\begin{align}\label{eq:conv-rho-mle}
			\lim_{\lambda\rt 0} \limsup_{n\to\infty} \mathcal{W}_{p'}\lf(q^{(n)}_{\hat \theta_{n,\lambda}}, \mathcal{A}\ri)=0.
		\end{align}
		For any compact set $K\subset \R^d$, 
		Then
		\begin{align}
			\label{eq:MLE-set-conv}
			\lim_{\lambda \rt 0} \limsup_{n\to\infty} \lf\|\bar b_{\hat \theta_{n,\lambda}} -\mathcal{M}_{\mrm{mle}}\ri\|_K=0.
		\end{align}
		%
		%
	\end{theorem}
	
	Thus the  estimators $\bar b_{\hat \theta_{n,\lambda}}$ converges to the the MLE set $\mathcal{M}_{\mrm{mle}}$ as the complexity of the parameteric class (e.g. complexity of the underlying NN) increases and the penalty parameter $\lambda$ decreases.

\looseness=-1	
\textbf{Monte-Carlo Approximation:} Now fix a NN architecture $\{f_\theta: \theta \in \Theta\}$. In practice, the loss function $\loss_{\lambda}(q_\theta)$ (where $q_\theta = (f_\theta)_{\#}\rfdist$) requires the evaluation of $\bar b_{\theta}(\cdot)  \equiv \bar\drft_{q_\theta}(\cdot) = \int_{\R^{d'}}b(\cdot, y)q_\theta(dy)$, and the integral is typically unavailable in closed form. We therefore use a Monte Carlo (MC) approximation at the observed states, $\bar\drft^{\mathrm{MC}}_\theta\bigl(x_m\bigr)
		=
		\frac{1}{L}\sum_{l=1}^L \drft \bigl(x_m, f_\theta(Z^{(l)})\bigr)$   where
	$
	Z^{(l)} \stackrel{i.i.d.}\sim \rfdist, \ l = 1,2,\ldots,L,
	$
Notice that by the law of large numbers,
$\bar\drft^{\mathrm{MC}}_\theta(x_m) \to \bar b_{\theta}(x_m)$ as $L \to \infty$. In fact, under mild regularity conditions, this convergence is uniform in $\theta$ (assuming $\Theta$ is compact). It can then be shown that the estimator $\hat\theta^{\mathrm{MC}}$, computed as 
{\setlength{\abovedisplayskip}{-0.2pt}
 \setlength{\belowdisplayskip}{-0.2pt}
$$\hat\theta^{\mathrm{MC}}
\dfeq
\argmin_{\theta \in \Theta}
\loss_\lambda\!\left(\bar\drft^{\mathrm{MC}}_\theta\right),$$
}
%
converges to $\hat \theta$ as $L$ increases.
Note that, in order to evaluate the log-likelihood \(\ln L_T(\bar\drft^{\mathrm{MC}}_\theta \mid \bx_{0:M_0})\), the MC approximation
	\(\bar\drft^{\mathrm{MC}}_\theta\) is required to be evaluated only at the finitely many observed states \(\bx_{0:M_0-1}\).  	An important advantage of the flow-based parametrization of \(q_\theta\) is that it enables the use of the {\em reparameterization trick} for gradient-based optimization. Since the samples \(Z^{(l)}\) are drawn from a fixed reference distribution \(\rfdist\) that does not depend on the NN parameter \(\theta\), gradients of
	\(\ln L_T(\bar\drft^{\mathrm{MC}}_\theta \mid \bx_{0:M_0})\)
	can be computed by differentiating through the Monte Carlo approximation of the drift.
	
	In particular, when computing $\nabla_\theta \ln L_{T}$, differentiation with respect to $\theta$ can be pushed inside the Monte Carlo sums, and the resulting derivatives act only through the map $f_\theta$. At the integrand level, this involves the chain rule
$
	\partial_\theta \drft\bigl(x, f_\theta(x,Z)\bigr)
	=
	\partial_y \drft(x,y)\big|_{y=f_\theta(x,Z)} \,
	\partial_\theta f_\theta(x,Z),
$
	This approach bypasses the need for the score-function estimator, which would require gradients of the density $\rho_\theta$ and typically suffers from high variance. By contrast, our formulation allows for high-precision, low-variance gradient estimates via standard backpropagation through the flow map $f_\theta$.

	\subsection{Variational Inference} 
 \looseness=-1
	While the  flow-based optimization described above performs well in many settings, principled uncertainty quantification requires a Bayesian formulation. To this end, we place a prior distribution $p_{\mathrm{prior}}(\theta)$ over the neural network parameters. Applying Bayes' rule yields the posterior distribution
	
	\begin{align}
		\label{eq:post}
		\post(\theta \mid \bx_{0:M_0})
		&= L_T(\bar b_\theta^{\mrm{MC}}\mid   \bx_{0:M_0}) 
			\,p_{\mathrm{prior}}(\theta) \Big/		
			p_{\mathrm{marginal}}(\bx_{0:M_0}), 
	\end{align}
	where the normalization constant $p_{\mathrm{marginal}}(\bx_{0:M_0})= \int L_T(\bar b_\theta^{\mrm{MC}}\mid   \bx_{0:M_0}) 
	\,p_{\mathrm{prior}}(\theta)
	\, d\theta$ does not depend on $\theta$.  	However, since the parameter $\theta$ typically lies in a high-dimensional space, standard MCMC methods become computationally expensive and often impractical for posterior approximation.
	We adopt a variational inference framework to obtain a scalable approximation of the posterior \(\post\) using a parametric family \(\SC{P}_1(\Theta)\) induced by a flow \(g_\vart\):
	\begin{align}
		\label{norm_flow}
		\SC{P}_1(\Theta)
		=
		\Big\{
		\rho \in \SC{P}(\Theta):
		\rho \equiv \rho_\vart = (g_\vart)_{\#}\rho_{\mathrm{ref}},
		\ \vart \in \SC{V}
		\Big\},
	\end{align}
	where \(g_\vart\) is modeled using a neural network. The objective is to determine the optimal parameter by minimizing the reverse KL-divergence between the true posterior and its variational approximation:
	\begin{align*}
		\vart^*
		=
		\argmin_{\vart\in\SC{V}}
		\mathrm{KL}\big(
		\rho_\vart(\cdot)
		\,\|\, 
		\post(\cdot\mid \bm{x}_{0:M_0})
		\big).
	\end{align*}
	Since the marginal likelihood does not depend on variational parameter \(\vart\), easy algebra shows that the above minimization problem is equivalent to maximizing the Evidence Lower Bound (ELBO):
	\begin{align}
		\mathrm{ELBO}(\vart)
		&=
		\int_{\Theta}
		\ln L_T(\bar b_\theta^{\mrm{MC}}\mid   \bx_{0:M_0}) 	\rho_{\vart}(\theta) d\theta	-
		\mathrm{KL}\big(
		\rho_\vart \,\|\, p_{\mathrm{prior}}
		\big).
	\end{align}
	Thus, the optimal parameter is obtained by
	\begin{align*}
		\vart^*
		=
		\argmax_{\vart\in\SC{V}}
		\mathrm{ELBO}(\vart).
	\end{align*}
	To optimize the ELBO, we employ the reparameterization trick. Since
	\(
	\rho_\vart = (g_\vart)_\# \rho_{\mathrm{ref}},
	\)
	sampling from \(\rho_\vart\) can be achieved by first drawing
	\(
	\xi \sim \rho_{\mathrm{ref}}
	\)
	and setting
	\(
	\theta = g_\vart(\xi).
	\)
	This allows us to approximate the ELBO using Monte-Carlo samples:
	\begin{align*}
		\mathrm{ELBO}(\vart)
		\approx
		\frac{1}{K}
		\sum_{k=1}^K
		\ln L_T\left(
		\bar\drft^{\mathrm{MC}}_{\theta^{(k)}} \mid   \bx_{0:M_0}\right)
		-
		\mathrm{KL}\big(
		\rho_\vart \,\|\, p_{\mathrm{prior}}
		\big),
		\quad \theta^{(k)} = g_\vart(\xi^{(k)}), \
		\xi^{(k)} \stackrel{\mrm{iid}}\sim \rho_{\mathrm{ref}} .
	\end{align*}
	This formulation enables efficient gradient-based optimization of the ELBO
	with respect to \(\vart\).

	\begin{algorithm}[H]
		\caption{Flow-based Variational Inference for Integral-Drift SDEs}
		\label{alg:flow-vi}
		\begin{algorithmic}[1]
			\Require Data $\bx_{0:M_0}$, kernel $b(x,y)$, flows $f_\theta$, $g_\vart$, prior $p_{\mathrm{prior}}$, sample sizes $L,K$
			\Ensure Variational parameter $\hat\vart$
			
			\State Initialize $\vart$
			
			\Repeat
			
			\State Sample $\xi^{(1)},\dots,\xi^{(K)} \sim \rho_{\mathrm{ref}}$  \hfill (Reparameterization)
			
			\State Compute
			$
			\theta^{(k)} = g_\vart(\xi^{(k)}), \quad k=1,\dots,K
			$
			
			\For{$k=1,\dots,K$}
			
			\State Sample $Z^{(1)},\dots,Z^{(L)} \sim \rfdist$
			
			\For{$m=0,\dots,M_0-1$}
			
			\State Compute Monte-Carlo drift: 
			$
			\bar b^{(\mathrm{MC},L)}_{\theta^{(k)}}(x_m)
			=
			\frac1L\sum_{l=1}^L
			b(x_m,f_{\theta^{(k)}}(Z^{(l)}))
			$
			
			\EndFor
			
			\EndFor
			
			\State Estimate Monte-Carlo ELBO
			
			\State Update $\vart$ using gradient ascent:
			$
			\vart \leftarrow
			\vart
			+
			\eta \nabla_\vart
			{\mathrm{ELBO}}(\vart)
			$
			
			\Until{convergence}
			
			\State \Return $\hat\vart$
			
		\end{algorithmic}
	\end{algorithm}

\section{Experiments} 

\textbf{Model:}	We consider a tagged particle interacting with a collection of solvent particles. The solvent configuration consists of $N$ particles with positions $y_1,\dots,y_N \in \mathbb{R}^d$, whose interactions are governed by the quadratic potential
{\setlength{\abovedisplayskip}{-2pt}
		\setlength{\belowdisplayskip}{-2pt}
\begin{align}
\label{eq:potential}
U_N(y_1,\dots,y_N)
=
\frac{a}{2}\sum_{i<j}\|y_i-y_j\|^2
+
\frac{\kappa}{2}\sum_{i=1}^N \|y_i\|^2,
\qquad a,\kappa>0.
\end{align}
}
\enlargethispage{2\baselineskip}	
This induces the following fast gradient type SDE \eqref{eq:fast_sde} for  $Y^{(n)}$ with drift 
$
\beta(y) = -\gamma \nabla_y U_N(y), \ \gamma>0,
$
and $\alpha(y) = \sqrt 2.$
The corresponding invariant distribution is the Gibbs measure
{\setlength{\abovedisplayskip}{3pt}
		\setlength{\belowdisplayskip}{-2pt}
\begin{align}
\label{eq:gibbs}
\inv_0(dy_1,\dots,dy_N)
\propto
\exp\bigl(-\gamma U_N(y_1,\dots,y_N)\bigr)\,dy_1\cdots dy_N.
\end{align}
}

The tagged particle $X^{(n)}$ interacts with the solvent through a pairwise potential $V$, yielding the drift
{\setlength{\abovedisplayskip}{-1pt}
 \setlength{\belowdisplayskip}{-1pt}
\begin{align*}
b(x,y_1,\dots,y_N)
=
-\sum_{i=1}^N \nabla_x V(\|x-y_i\|).
\end{align*}
Under scale separation, the solvent rapidly equilibrates, and $X_{\inv_0}$ satisfies SDE \eqref{eq:inv-dist-pde} with $\bar b_{\inv_0}$.
}

{\bf Experimental Design:} 	Synthetic data are generated by simulating the  multiscale SDE system given by \eqref{eq:slow_sde} and \eqref{eq:fast_sde} with the above drift and diffusion functions
 This ground truth $\inv_0$is only used for evaluating estimation accuracy. 	
	The SDE system is simulated over a fixed time horizon $[0,5]$ using a fine
	discretization to obtain high-frequency observations. Importantly, evaluation of true averaged SDE with the estimated one is done beyond the observed time-interval
	
\textbf{Computational Resources:}
Experiments were performed on an HPC node with dual Intel Xeon Gold 6342 CPUs and an NVIDIA A2 16\,GB GPU using PyTorch.

\textbf{Comparison:}	We compare the proposed structured estimator with a direct (unstructured) NN model for $\bar b_{\inv_0}$ that ignores the integral representation 
$\bar b_{\inv_0}(x) = \int b(x, y)\,\inv_0(dy)$. 
This NN is trained by maximizing the likelihood induced by the Euler--Maruyama discretization of the corresponding SDE.

	\begin{table}[H]
		\label{tab:MSE}
		\centering
		\setlength{\tabcolsep}{10pt}
		\renewcommand{\arraystretch}{1.3}
		
		\begin{tabular}{|c||c|c|c|c|c|}
			\hline
			$d$ & 1 & 1 & 1 & 2 & 2 \\
			\hline
			$N$ & 10 & 15 & 20 & 10 & 15 \\
			\hline
			$n=100$  & 0.02 & 0.0008 & 0.0053 & 0.0058 & 0.025 \\
			$n=1000$ & 0.0029 & 0.0007 & 0.0004 & 0.0027 & 0.0035 \\
			$n=15000$ & 0.0002 & 0.0002 & 0.0001 & 0.004 & 0.0055 \\
			\hline
		\end{tabular}
		\vs{.2cm}
		\caption{Mean squared errors for different datasets}
	\end{table}
	
\vs{-0.7cm}
\textbf{Results:} Table \ref{tab:MSE} reports the mean squared error (MSE) between the estimated and true averaged drift functions, evaluated on a grid of state values, and demonstrates the accuracy of the proposed method. In addition, we provide visual comparisons of the estimated and true drift functions, and plot trajectories of the true slow process alongside those generated by the SDE driven by the learned averaged drift (with a fixed random seed), both within and beyond the observation time window. 
The close agreement between these trajectories, even outside the observation window, is particularly noteworthy, as it demonstrates the predictive capability of the learned model beyond the range of observed data. This further indicates that the accuracy of the estimated averaged drift $\hat{\bar{b}}_{\inv_0}$ arises from genuine learning by the proposed method rather than overfitting.  In contrast, the unstructured neural network model performs noticeably worse; for the first setting ($d=1$, $N=10$), it achieves MSE values of $0.014$ and $0.008$ for $n=1000$ and $n=15000$, respectively, which are approximately 10 times  larger than those of the structured flow-based estimator.


	\begin{figure}[!htb]
		\centering
		\setlength{\abovecaptionskip}{2pt}
		\setlength{\belowcaptionskip}{0pt}
		
		\begin{minipage}{0.3\textwidth}
			\centering
			\includegraphics[width=0.97\linewidth]{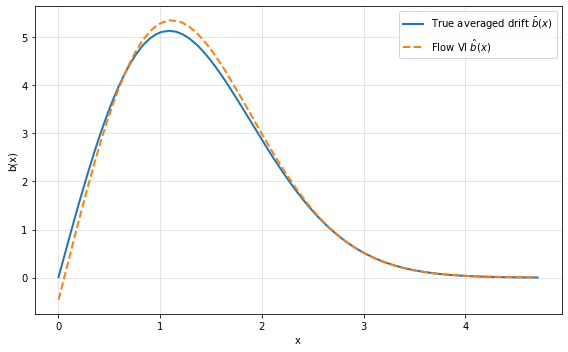}
			
			{\scriptsize $n=100$}
		\end{minipage}
		\hfill
		\begin{minipage}{0.3\textwidth}
			\centering
			\includegraphics[width=0.97\linewidth]{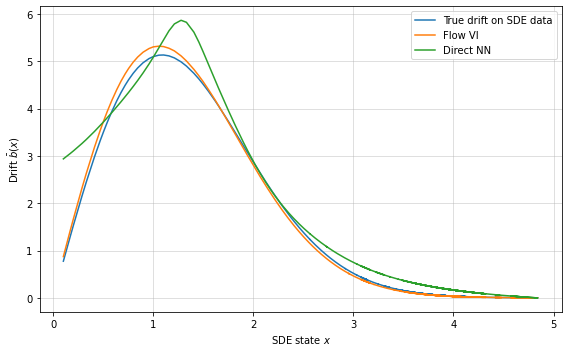}
			
			{\scriptsize Benchmark with a neural network}
		\end{minipage}
		\hfill
		\begin{minipage}{0.3\textwidth}
			\centering
			\includegraphics[width=0.97\linewidth]{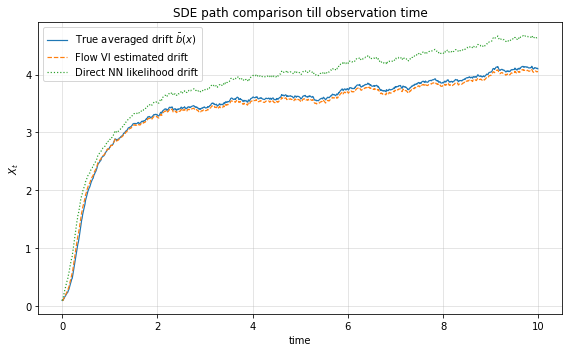}
			
			{\scriptsize SDE path comparison till observation time}
		\end{minipage}
		
		\vspace{0.05cm}
		
		\begin{minipage}{0.3\textwidth}
			\centering
			\includegraphics[width=0.97\linewidth]{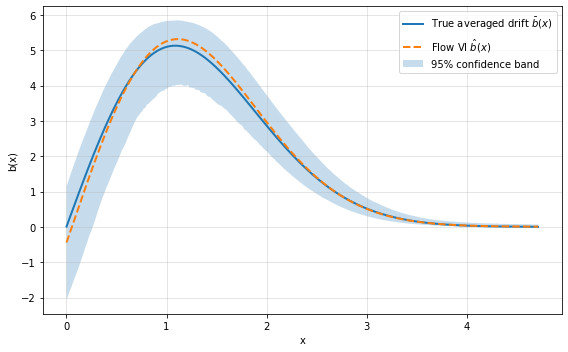}
			
			{\scriptsize CB, $n=100$}
		\end{minipage}
		\hfill
		\begin{minipage}{0.3\textwidth}
			\centering
			\includegraphics[width=0.97\linewidth]{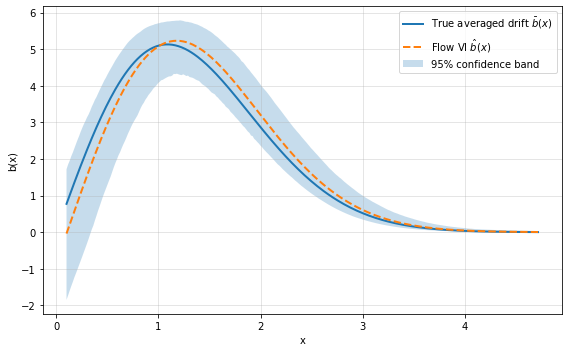}
			
			{\scriptsize CB, $n=1000$}
		\end{minipage}
		\hfill
		\begin{minipage}{0.3\textwidth}
			\centering
			\includegraphics[width=0.97\linewidth]{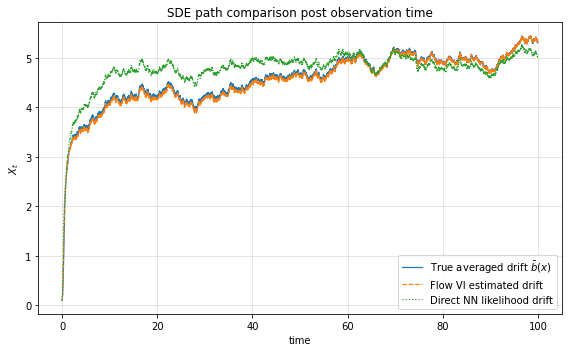}
			
			{\scriptsize SDE path comparison post observation time}
		\end{minipage}
		
		
		\caption{One-dimensional drift recovery with confidence bands.}
		\label{fig:onedim-drift-recovery}
	\end{figure}
	
	\begin{figure}[H]
		\centering
		\setlength{\abovecaptionskip}{2pt}
		\setlength{\belowcaptionskip}{0pt}
		
		\begin{minipage}{0.30\textwidth}
			\centering
			\includegraphics[width=0.9\linewidth]{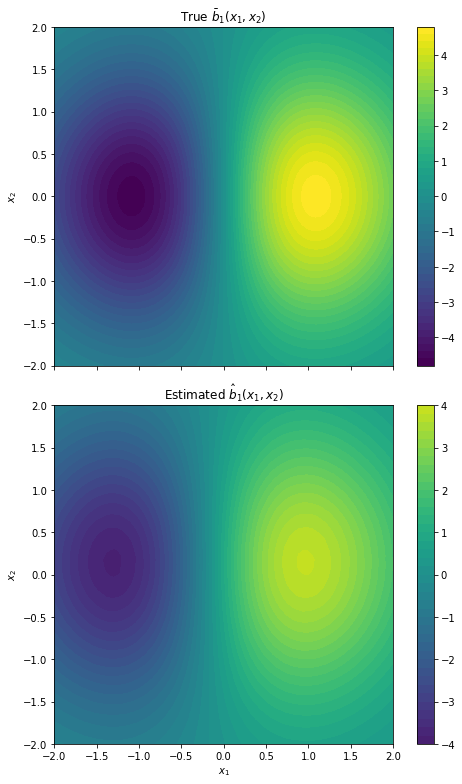}
			
			{\scriptsize $b_1$, $n=100$}
		\end{minipage}
		\hfill
		\begin{minipage}{0.30\textwidth}
			\centering
			\includegraphics[width=0.9\linewidth]{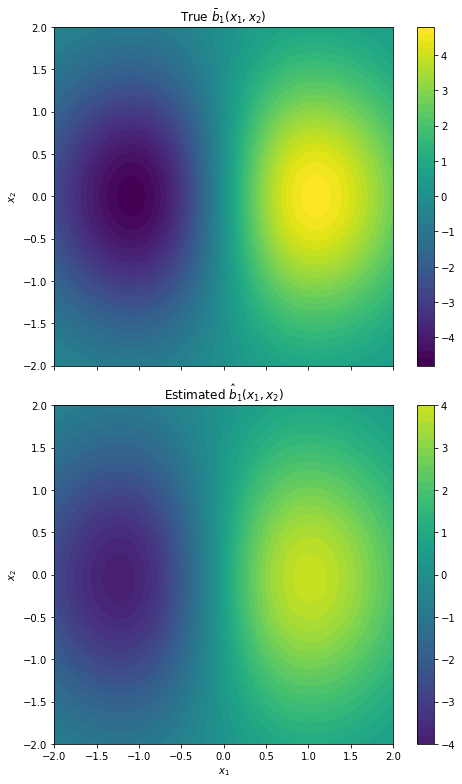}
			
			{\scriptsize $b_1$, $n=1000$}
		\end{minipage}
		\hfill
		\begin{minipage}{0.30\textwidth}
			\centering
			\includegraphics[width=0.9\linewidth]{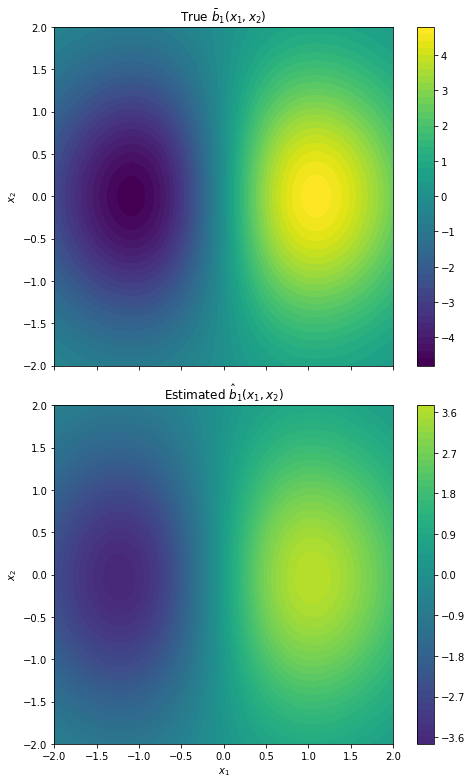}
			
			{\scriptsize $b_1$, $n=15000$}
		\end{minipage}
		
		\vspace{0.15cm}

		\begin{minipage}{0.30\textwidth}
			\centering
			\includegraphics[width=0.9\linewidth]{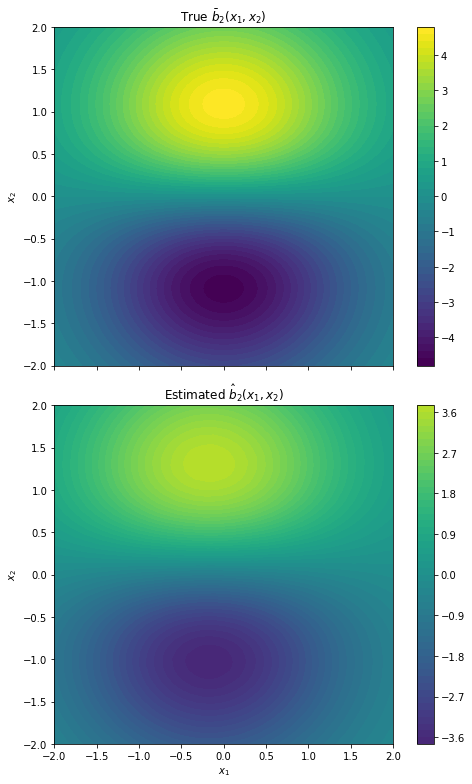}
			
			{\scriptsize $b_2$, $n=100$}
		\end{minipage}
		\hfill
		\begin{minipage}{0.30\textwidth}
			\centering
			\includegraphics[width=0.9\linewidth]{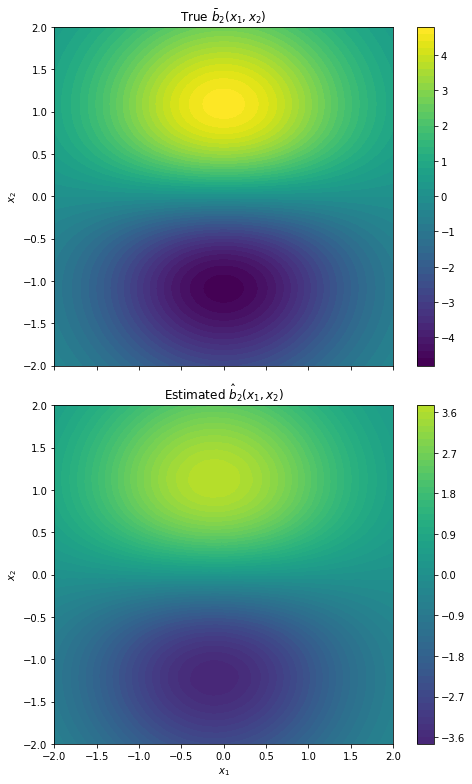}
			
			{\scriptsize $b_2$, $n=1000$}
		\end{minipage}
		\hfill
		\begin{minipage}{0.30\textwidth}
			\centering
			\includegraphics[width=0.9\linewidth]{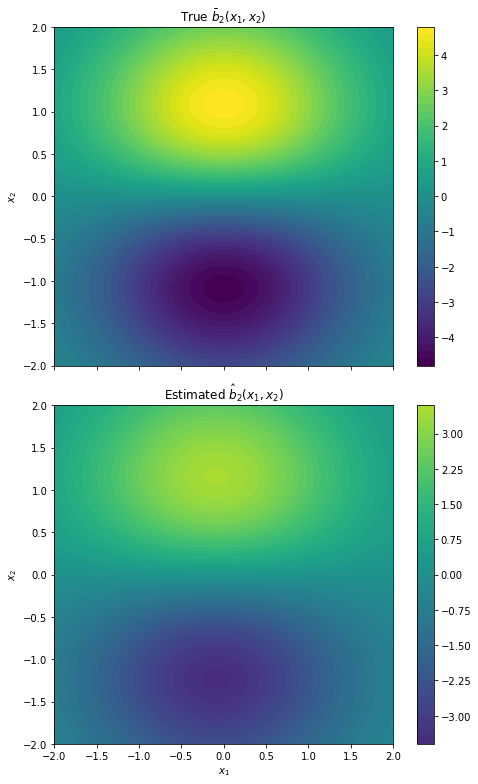}
			
			{\scriptsize $b_2$, $n=15000$}
		\end{minipage}
		
		\vspace{0.1cm}
		
		\caption{Two-dimensional drift recovery for the two drift components. The first row shows recovery of $b_1$, and the second row shows recovery of $b_2$, for increasing values of the scale parameter $n$.}
		\label{fig:twodim-drift-recovery-b1-b2}
	\end{figure}
	\section{Conclusion}
We proposed a data-driven framework for learning effective dynamics in partially observed multiscale stochastic systems. A principled stochastic averaging reduction yields a lower-dimensional averaged SDE for the slow variables, where the drift depends on the invariant distribution of the latent fast process. We formulate inference as a latent-measure learning problem, parameterize this invariant measure using a normalizing flow, and train the resulting model via a likelihood objective based on reduced model.
We show that penalized minimizers over increasingly expressive neural pushforward classes approach the global MLE under suitable assumptions. We further introduce a flow-based variational posterior to quantify uncertainty in the learned effective dynamics.

\looseness=-1
The performance of the method critically depends on the expressivity of the normalizing flow. The proposed approach, as described, does not directly apply to strongly coupled multiscale systems where the fast dynamics depend on the slow variables, nor to settings with sparse and noisy observations of the slow component. Extending the method to these regimes, together with the development of more expressive architectures based on continuous-time flows such as neural ODEs, constitutes an important direction for future work.

	\bibliography{Ref-ML-new}
	\bibliographystyle{plain}

	
	\appendix
	
	\section{Theoretical details and results}
	\subsection{Stochastic averaging of fast-slow sde system}
	
	\begin{proof}[\bf Proof of Theorem \ref{th:sto-avg}]
		Notice that $Y^{(n)}(\cdot) \eqd Y(n\ \cdot)$ where $Y$ satisfies the SDE:
		\begin{align}
			dY(t) =  \beta\bigl(Y (t)\bigr)\,dt	+ \alpha\bigl(Y(t)\bigr)\,d\tilde W(t).
		\end{align}
		Fix $p>1$. By It\^o's lemma 
		\begin{align*}
			V^p(Y(t)) =&\ V^p(y_0) + p \int_0^t V^{p-1}(Y(s)) \mathcal{A}^1(Y(s)) ds\\
			& \hs{.1cm} + p(p-1)\int_0^t V^{p-2}(Y(s)) \nabla^\top V(Y(s))(\alpha\alpha^\top(Y(s)))\nabla V(Y(s))\ ds+ M_p(t),
		\end{align*}
		where $M_p(t) \dfeq p \int_0^t V^{p}(Y(s))\nabla^\top V(Y(s))\alpha(Y(s))d\tilde W(s)$ is a martingale.	Now by Assumption \ref{assum:sto-avg}, for any $K>0$, there exists an $R_K>0$ such that for $\|y\|>R_0$, 
		$$\mathcal{A}^1 V(y) \leq - c_*/2, \quad \nabla^\top V(y)(\alpha\alpha^\top(y))\nabla V(y) \leq  \f{1}{K} V(y) |\mathcal{A}^1 V(y)| \equiv   -\f{1}{K} V(y) \mathcal{A}^1 V(y).$$
		By splitting the second integral according as $\|Y(s)\| > R_K$ or not we get for some constant $C_p$
		\begin{align*}
			V^p(Y(t)) \leq&\ V^p(y_0) + C_p t  + \lf(p - p(p-1)/K\ri) \int_0^t V^{p-1}(Y(s))\mathcal{A}^1(Y(s)) \indic_{\{\|Y(s)\| >R_K\}} ds\\
			& \ + M_p(t).
		\end{align*}
		Choosing $K=2p$, and using Assumption \ref{assum:sto-avg} it follows that
		\begin{align*}
			\sup_t \f{1}{t} \int_0^t \EE V^p(Y(s)) \indic_{\{\|Y(s)\| >R_K\}} ds < \infty. 
		\end{align*}	
		It follows that for any $t>0$ 
		\begin{align}\label{eq:lyapu-bd}
			\sup_t \f{1}{t} \int_0^{t} \EE V^p(Y(s)) ds \equiv  \sup_t  \EE \int_{\R^d} V^p(y) \Gamma_t(dy) < \infty,
		\end{align}
		where the occupation measure $\{\Gamma_t\}$ of $Y$ is  defined by $\Gamma_t(A) = \f{1}{t}\int _0^{t} \indic_{\{Y(s) \in A\}}\ ds$. 
		It  follows that  $\{\Gamma_t\}$ is tight as $\mathcal{P}(\R^{d'})$-valued random variables. If $\eta$ is a limit point of $\{\Gamma_t\}$, then by the Krylov-Bogoliubov theorem $\eta$ is a stationary distribution of $Y$. Assumption \ref{assum:sto-avg}:(iv) implies that  $Y$ is irreducible and hence admits a unique stationary distribution $\pi_0$. Therefore, we have $\eta = \pi_0$, and thus $\Gamma_t \stackrel{t \rt \infty} \Rt \inv_0$ in probability in the space $\mathcal{P}(\R^{d'})$. Furthermore, since \eqref{eq:lyapu-bd} holds, for any function $f$ satisfying $f=O(V^p)$ for some $p>0$, $\f{1}{t} \int_0^t f(Y(s)) ds =  \int f(y) \Gamma_t(dy) \stackrel{t \rt \infty} \Rt \int f(y) \inv_0(dy).$
		
		Now by a simple change of variable formula we have for any (measurable) function $f$, $\int_0^t f(Y^{(n)}(s))ds \eqd \f{1}{n}\int_0^{nt} f(Y(s))\ ds$. 
		and thus for any fixed $T>0$
		Define the random measure $\Gamma^{(n)}$ on $\R^d \times [0,t]$ as $\Gamma^{(n)} (A \times [0,t]) = \int_{\R^d \times [0,t]} \indic_{\{Y^{(n)}(s) \in A\}} ds$. It follows that $\Gamma^{(n)} \stackrel{n \rt \infty} \Rt\inv_0 \ot \mrm{Leb}$  and that for any continuous $g: \R^{d'} \times [0,T] \rt \R^{d"}$ satisfying $\sup_{s\leq T}\|g(y,s)\| \leq C_g V^p(s)$ for some $p>0$,  we have
		\begin{align}\label{eq:conv-occ}
			\int_{\R^d \times [0,t]}  g(y,s) \Gamma^{(n)}(dy\times ds) \equiv \int_{\R^d \times [0,t]} g(Y^{(n)}(s),s)ds \stackrel{n \rt \infty} \Rt \int_{\R^d \times [0,t]} f(y,s) \inv_0(dy)ds.
		\end{align}
		
		Next, simple estimates show that $\sup_{0 \leq t, t+h \leq T}\EE(\|X^{(n)}(t+h) - X^{(n)}(t)\|^2 = O(h^\alpha)$ from which it follows that the sequence $\{X_n\}$ is tight in $C([0,T], \R^d)$. Let $ X$ is a limit point of $\{X_n\}$. By Skorohord representation theorem we can assume without loss of generality that $X_{n}  \Rt X$ a.s. in $C([0,T], \R^d)$  along a subsequence, which for notational convenience we continue to index by $n$. We next show that 
		\begin{align}\label{eq:conv-drft}
			\int_{0}^t (b(X_n(s), Y_n(s)) ds \stackrel{n \rt \infty} \Rt \int_{\R^{d'}\times [0,t]} (b(X(s), y) \inv_0(dy) ds = \bar b_{\inv_0}(X(s)).
		\end{align}
		To this end, write
		\begin{align*}
			\int_{0}^t (b(X_n(s), Y_n(s)) ds =&  \int_{\R^{d'}\times[0,t]} (b(X_n(s), y) \Gamma^{n}(dy \times ds) 
			= I_1(t)+I_2(t),   
		\end{align*}
		where by Assumption \ref{assum:sto-avg}, \eqref{eq:conv-occ} as $n \rt \infty$
		\begin{align*}
			\sup_{t\leq T}|I_1(t)| = &\ \sup_{t\leq T}\lf|\int_0^t  ((b(X_n(s), Y_n(s)) - b(X(s),Y_n(s))) ds\ri|\\
			\leq&\  C_0 \sup_{t\leq T} |X_n(t) - X(t)|^{\alpha_0} \int_0^T V^{p_0}(Y^{(n)}(s))ds \stackrel{\PP} \Rt 0. 
		\end{align*}
		Next by \eqref{eq:conv-occ}
		\begin{align*}
			I_2(t) = \int_0^t  b(X_n(s), Y_n(s)) ds &\ \equiv  \int_{\R^{d'}\times [0,t]} (b(X(s), y) \Gamma^{n}(dy \times ds)\\
			& \  \stackrel{n \rt \infty} \Rt \int_{\R^{d'}\times [0,t]} (b(X(s), y) \inv_0(dy) ds = \bar b_{\inv_0}(X(s)),
		\end{align*}
		which proves \eqref{eq:conv-drft}. Similarly, by Burkholder-Davis-Gundy inequality it follows that 
		\begin{align*}
			E\lf[\sup_{t\leq T}\lf|\int_0^t  ((\sigma(X_n(s), Y_n(s)) - \sigma(X(s),Y_n(s))) dW(s)\ri|^2\ri] \stackrel{n \rt \infty} \Rt 0.
		\end{align*}
		Notice that for the martingale term $M^{(n)}(\cdot) = \int_0^\cdot \sigma(X(s), Y_n(s)) dW(s)$ has the quadratic variation given by
		$$[M^{(n)}]_t = \int_0^t \sigma\sigma^\top (X(s), Y_n(s)) ds  \stackrel{n \rt \infty} \Rt \int_{\R^{d'}\times [0,t]} \sigma\sigma^\top (X(s), y) \inv_0(dy) ds \equiv  \bar{\sigma}^2_{\inv_0}(X(s)).$$
		It follows by the Martingale CLT that  as $n\rt \infty$
		\begin{align}\label{eq:conv-drft}
			M^{(n)}(t) \stackrel{\mrm{Law}} \Rt \bar{\sigma}_{\inv_0}(X(s)).
		\end{align}
		Therefore the limit point $X$ satisfies \eqref{eq:int-drift-sde}, and since this SDE admits a unique solution the result follows.
		%
	\end{proof}		
	
	\subsection{Wasserstein convergence of neural pushforwards}
	\begin{proof}[Proof of Lemma \ref{lem:wass-density-flows}]
		Since 
		\(
		\nu_{\mathrm{ref}}, \pi_0 \in \mathcal P^{(p)}_{\mathrm{abs}}(\R^{d'})
		\),
		there exists a measurable transport map 
		\(
		T:\R^{d'} \to \R^{d'}
		\)
		such that
		\[
		T_{\#}\nu_{\mathrm{ref}} = \pi_0 .
		\]
		This follows, for instance, from the Knothe--Rosenblatt rearrangement or Brenier's theorem under absolute continuity.
		
		Next, we use the universal approximation theorem in $L^p(\mu)$ for general measures.  
		By \cite{hornik1991approximation}, finite neural networks with non-polynomial activation functions are dense in $L^p(\mu)$ for any finite measure $\mu$ (also see \cite{KiLy20}).  
		Taking $\mu = \nu_{\mathrm{ref}}$, there exists a finite neural network with non-polynomial activation such that,
		\(
		f:\R^{d'}\to\R^{d'}
		\)
		such that
		\[
		\|f - T\|_{L^p(\nu_{\mathrm{ref}})} 
		=
		\left(
		\int_{\R^{d'}} 
		\|f(x)-T(x)\|^p
		\,\nu_{\mathrm{ref}}(dx)
		\right)^{1/p}
		< \varepsilon .
		\]
		Define the coupling
		\[
		\gamma := (T,f)_{\#}\nu_{\mathrm{ref}} .
		\]
		Then $\gamma$ is a valid coupling between 
		\(
		T_{\#}\nu_{\mathrm{ref}} = \pi_0
		\)
		and 
		\(
		f_{\#}\nu_{\mathrm{ref}}
		\).
		Therefore, by definition of Wasserstein distance,
		\[
		W_p^p\bigl(\pi_0, f_{\#}\nu_{\mathrm{ref}}\bigr)
		\le
		\int_{\R^{d'}}
		\|T(x)-f(x)\|^p
		\,\nu_{\mathrm{ref}}(dx).
		\]
		Taking the $p$th root yields
		\[
		W_p\bigl(\pi_0, f_{\#}\nu_{\mathrm{ref}}\bigr)
		\le
		\|f-T\|_{L^p(\nu_{\mathrm{ref}})}
		< \varepsilon .
		\]
		
		This completes the proof.
		
	\end{proof}
	\subsection{Penalized Likelihood Convergence}
	\begin{lemma}
		\label{lem:cmpt-wass}
		Let $p>1$ and  $C>0$. The set $ \scr{P}^{(p)}_C\dfeq \lf\{\eta \in \mathcal{P}(\R^{d'}): \mfk{m}_p(\eta) \equiv  \int \|y\|^p \eta(dy) \leq C\ri\}$ is compact in weak topology of $\mathcal{P}(\R^{d'})$. If $\tilde \eta$ is a limit point of $\scr{P}^{(p)}_C$ with $\scr{P}^{(p)}_C \ni \eta_n \stackrel{n \rt \infty}\lRT \tilde \eta$, then $\mathcal{W}_{p'}(\eta_{n}, \eta) \stackrel{n \rt \infty}\Rt 0$ for any $1\leq p'<p$. In particular, $ \scr{P}^{(p)}_C$ is also compact in $(\mathcal{P}(\R^{d'}), \mathcal{W}_{p'})$ for any $1\leq p'<p$.
	\end{lemma}
	\begin{proof}[Proof of Lemma \ref{lem:cmpt-wass}]
		This is a standard consequence of Prokhorov's theorem together with bounded 
		$p$-moment compactness in Wasserstein space; see 
		\cite[Theorem 6.9]{villani2009optimal}.
	\end{proof}
	\begin{proof}[Proof of Theorem \ref{thm:pen-mle-conv}]
		The following observation is crucial to the proof. Suppose $\{\eta_n\} \subset \scr{P}^{(p)}_C $ and for some $p' \in [1,p]$\begin{align}
			\label{eq:conv-eta}
			\mathcal{W}_{p'}(\eta_n,\eta) \stackrel{n \rt \infty} \Rt 0.
		\end{align}
		
		\begin{itemize}[leftmargin=*]
			\item Notice that \eqref{eq:lip-b-cond} implies that for each $x$, $b(x, \cdot)$ satisfies the following growth condition: $b(x, y) = O(\|y\|^{q_0+1})$. Hence, if $p' \geq q_0+1$, then for any $x \in \R^d$,
			$$\bar b_{\eta_n}(x) \equiv \int b(x,y)\eta_n(dy) \stackrel{n \rt \infty} \Rt \bar b_{\eta} (x) \equiv \int b(x,y)\eta_n(dy),$$
			whence it follows that for $q_0+1\leq p'<p$ 
			\begin{align}
				\label{eq:conv-like}
				\lim_{n\rt \infty}\ln L_T(\bar b_{\eta_n} \mid  \bx_{0:M_0})  =\ln L_T(\bar b_{\eta} \mid  \bx_{0:M_0}),  \quad \text{ and } \  \liminf_{n \rt \infty} \loss_{\lambda}(\eta_n) \geq \loss_{\lambda}(\eta) 
.			\end{align}
			The second part of \eqref{eq:conv-like} uses lower semicontinuity of the map $\eta \rt \mfk{m}_p(\eta).$ Now observe that if $\{\lambda_n\}$ is a sequence such that $\lambda_n \rt 0$, then the first part of \eqref{eq:conv-like} and the fact that $\sup_n \mfk{m}_p(\eta_n) \leq C$ shows that 
			\begin{align}
				\label{eq:loss-conv-lambda}
				\lim_{n \rt \infty} \loss_{\lambda_n}(\eta_n) = -\ln L_T(\bar b_{\eta} \mid  \bx_{0:M_0})
.			\end{align}
			
			\item If in fact $\mathcal{W}_{p}(\eta_n,\eta) \stackrel{n \rt \infty} \Rt 0$ (i.e., $p'=p$), then $\mfk{m}_p(\eta_n) \rt \mfk{m}_p(\eta)$, and hence
			\begin{align}
				\label{eq:conv-loss}
				\lim_{n \rt \infty} \loss_{\lambda}(\eta_n) = \loss_{\lambda}(\eta). 
			\end{align}
			
			\item Suppose the convergence in \eqref{eq:conv-eta} holds for some $p/(p-q_0) \leq p'\leq p$. Then notice that $q_0p'/(p'-1) \leq p$, and hence
			$$\sup_n \mfk{m}_{q_0p'/(p'-1)}(\eta_n) \vee \mfk{m}_{q_0p'/(p'-1)}(\eta) < \infty.$$
			Now recalling that $C_b:\R^d \rt [0,\infty)$ in \eqref{eq:lip-b-cond} is locally bounded, we have from Lemma \ref{lem:cmpt-wass} for any compact set $K \subset \R^d$
			\begin{align}
				\label{eq:drft-conv}
				\begin{aligned}
					& \|\bar b_{\eta_n} - \bar b_{\eta}\|_K \equiv \sup_{x \in K}|\bar b_{\eta_n}(x) -\bar b_{\eta}(x)|\\
					& \ \leq 3^{1/p'} \sup_{x \in K} C_b(x)  \lf(1+ \sup_n\mfk{m}_{q_0p'/(p'-1)}(\eta_n)+ \mfk{m}_{q_0p'/(p'-1)}(\eta)\ri)^{(p'-1)/p'}\!\!\!\! \!\!\!\mathcal{W}_{p'}(\eta_n,\eta)\ \stackrel{n \rt \infty} \Rt 0.
				\end{aligned}
			\end{align}

		\end{itemize}
		
		We next establish the following claim.
		
		{\em Claim:} Fix any  $\hat \rho_{\lambda} \in \mathcal{A}_\lambda \dfeq \argmin_{\rho \in \mathcal{P}^{(p)}(\R^{d'})}  \loss_{\lambda}(\rho)$. Then for any $1\leq p' < p$, $\lim_{\lambda \rt 0} \mathcal{W}_{p'}(\hat{\rho}_\lambda,\mathcal{A}) =0.$ Furthermore, for any compact set $K \subset \R^d$, $\lim_{\lambda \rt 0} \lf\|\bar b_{\hat{\rho}_\lambda}- \mathcal{M}_{\mrm{mle}}\ri\|_K=0.$
		
		It is easy to see that $\mathcal{A}_\lambda \neq \emptyset.$ Now notice that by the definition of $\hat\rho_*$ (see \eqref{eq:agmax-like}), $\hat \rho_{\lambda}$ 
		\begin{align}\label{eq:loss-ineq-0}
			-\ln L_T( \bar\drft_{\hat\rho_*} \mid \bx_{0:M_0}) \leq -\ln L_T( \bar\drft_{\hat\rho_\lambda} \mid \bx_{0:M_0}) \leq \loss_{\lambda}(\hat\rho_\lambda) \leq \loss_{\lambda}(\hat\rho_*) \equiv -\ln L_T( \bar\drft_{\hat\rho_*}\mid \bx_{0:M_0}) + \lambda \mfk{m}_p(\hat\rho_*).
		\end{align}
		This implies
		$$0 \leq \loss_{\lambda}(\hat\rho_\lambda)+\ln L_T( \bar\drft_{\hat\rho_*}\mid \bx_{0:M_0}) \equiv \lf(\ln L_T( \bar\drft_{\hat\rho_*}\mid \bx_{0:M_0})-\ln L_T( \bar\drft_{\hat\rho_\lambda}\mid \bx_{0:M_0})\ri)+ \lambda \mfk{m}_p(\hat\rho_\lambda) \leq \lambda \mfk{m}_p(\hat\rho_*).$$
		Since $\ln L_T( \bar\drft_{\hat\rho_*}\mid \bx_{0:M_0})-\ln L_T( \bar\drft_{\hat\rho_\lambda}\mid \bx_{0:M_0}) \geq 0$, it follows that $\lambda \mfk{m}_p(\hat\rho_\lambda) \leq \lambda \mfk{m}_p(\hat\rho_*)$, i.e., for any $\lambda>0$, $ \mfk{m}_p(\hat\rho_\lambda) \leq  \mfk{m}_p(\hat\rho_*).$ By Lemma \ref{lem:cmpt-wass}, the family $\{\hat\rho_\lambda\}$ is relatively compact in weak topology of $\mathcal{P}(\R^{d'})$ and $(\mathcal{P}(\R^{d'}), \mathcal{W}_{p'})$ for any $p'<p$.  If $\tilde \rho_*$ is one of its limit points and $\{\lambda_n\} \rt 0$ is a subsequence such that $\hat\rho_{\lambda_n} \stackrel{k\rt \infty} \lRT \tilde \rho_*$ then for any $p'<p$,
		\begin{align}\label{eq:rho-lambda-conv}
			\lim_{k\rt \infty}\mathcal{W}_{p'}\lf(\hat\rho_{\lambda_n}, \tilde \rho_*\ri) =0
		\end{align}
		Since, in particular,  \eqref{eq:rho-lambda-conv} holds for $q_0+1 \leq p' <p$, it follows by \eqref{eq:loss-conv-lambda} that $\lim_{n \rt \infty}\loss_{\lambda_n}(\hat\rho_{\lambda_n}) = -\ln L_T(\bar b_{\tilde \rho_*} \mid  \bx_{0:M_0})$. 
		Consequently, taking $\lambda_n \rt 0$ we get from \eqref{eq:loss-ineq-0}, 
		$$\ln L_T(\bar b_{\tilde \rho_*} \mid  \bx_{0:M_0}) = \ln L_T(\bar b_{\hat \rho_*} \mid  \bx_{0:M_0}) \ \implies \ \tilde \rho_* \in \mathcal{A} \ \implies \ \lim_{n \rt \infty} \mathcal{W}_{p'}(\hat\rho_{\lambda_n},\mathcal{A}) = 0.$$
		Thus $\bar b_{\tilde \rho_*} \in \mathcal{M}_{\mrm{mle}}$, and since \eqref{eq:rho-lambda-conv}, in particular,  holds for $p/(p-q_0) \leq p'< p$, we have by \eqref{eq:drft-conv}
		\begin{align*}
			\|\bar b_{\hat{\rho}_{\lambda_n}}- \mathcal{M}_{\mrm{mle}}\|_K \leq \sup_{x \in K}|\bar b_{\hat\rho_{\lambda_n}}(x) -\bar b_{\tilde \rho_*}(x)| \stackrel{n\rt \infty}\Rt 0.
		\end{align*}
		As these  limits hold independent of the choice of the specific subsequence $\{\lambda_n\}$, the claim follows.

		We next show that $\{q^{(n)}_{\hat \theta_{n,\lambda}}\}$ in Theorem \ref{thm:pen-mle-conv} is relatively compact in $(\mathcal{P}(\R^{d'}), \mathcal{W}_{p'})$ for any $p'<p$. 
		
		To this end, first notice that 
		$\ln L_T( \bar\drft_\rho \mid \bx_{0:M_0}) \leq C_0(\bx_{0:M_0})$, where the constant
		$C_0(\bx_{0:M_0}) \equiv -\lf(M_0d\ln(2\Delta \pi)/2+\f{1}{2}\sum_{m=0}^{M_0-1} \ln \lf(\det  A(x_m)\ri)\ri),$
		does not depend on $\rho$. It follows that
		\begin{align}\label{eq:loss-ineq-1}
			\loss_{\lambda}(\rho) \geq -C_0(\bx_{0:M_0}) + \l \mfk{m}_p(\rho) \geq  -C_0(\bx_{0:M_0}),
		\end{align}
		and hence 
		$$\loss^*_{\lambda} \dfeq  \inf_{\rho \in \mathcal{P}^{(p)}(\R^{d'})} \loss_{\lambda}(\rho) \in (-\infty,\infty).$$ 
		Fix an $\vep>0$, and let $\rho_{\lambda, \vep} \in \mathcal{P}^{(p)}(\R^{d'})$ be such that $\loss_{\lambda}(\rho_{\lambda, \ep}) \leq \loss^*_{\lambda}+\vep$. By Lemma \ref{lem:wass-density-flows}, there exists a sequence $\{\tilde \theta_{n,\lambda,\vep}\}$  such that
		$$\mathcal{W}_p\lf(q^{(n)}_{\tilde \theta_{n,\lambda,\vep}}, \rho_{\lambda, \vep}\ri) \stackrel{n \rt \infty}\Rt 0.$$
		It follows from \eqref{eq:conv-loss}, \eqref{eq:loss-ineq-1} and by the definition of $\hat \theta_{n,\lambda}$ (see \eqref{eq:NN-para-min}) that there exists $N_0$ such that for all $n \geq N_0$
		\begin{align} \label{eq:loss-lam-ineq}
			\begin{aligned}
				&-C_0(\bx_{0:M_0}) + \lambda \mfk{m}_p\lf(q^{(n)}_{\hat \theta_{n,\lambda}}\ri)\leq  \loss_{\lambda}\lf(q^{(n)}_{\hat \theta_{n,\lambda}}\ri) \leq \loss_{\lambda}\lf(q^{(n)}_{\tilde \theta_{n,\lambda, \vep}}\ri) \leq \loss^*_{\lambda}+2\vep\\
				& \implies  \mfk{m}_p\lf(q^{(n)}_{\hat \theta_{n,\lambda}}\ri) \leq \l^{-1}\lf(\loss^*_{\lambda}+2\vep+C_0(\bx_{0:M_0})\ri).
			\end{aligned}
		\end{align}
		By Lemma \ref{lem:cmpt-wass}, the sequence $\{q^{(n)}_{\hat \theta_{n,\lambda}}\}$ is relatively compact in weak topology of $\mathcal{P}(\R^{d'})$ and $(\mathcal{P}(\R^{d'}), \mathcal{W}_{p'})$ for any $p'<p$.  If $q_{\lambda,*}$ is one of its limit points and $\{q^{(n_k)}_{\hat \theta_{n_k,\lambda}} \}$ is a subsequence such that $q^{(n_k)}_{\hat \theta_{n_k,\lambda}} \stackrel{k\rt \infty} \lRT q_{\lambda,*}$ then for any $p'<p$
		\begin{align}
			\label{eq:wass-conv-flow-est}
			\lim_{k\rt \infty}\mathcal{W}_{p'}\lf(q^{(n_k)}_{\hat \theta_{n_k,\lambda}}, q_{\lambda,*}\ri) =0.
		\end{align}
		In particular, \eqref{eq:wass-conv-flow-est} holds for $q_0+1\leq p'< p$, and hence by \eqref{eq:loss-lam-ineq}, the second part of \eqref{eq:conv-like} and the definition of $\loss^*_{\lambda}$ that
		$$\loss^*_{\lambda} \leq \loss_\lambda(q_{\lambda,*}) \leq	\liminf_k \loss_\lambda \lf(q^{(n_k)}_{\hat \theta_{n_k,\lambda}}\ri) \leq \loss^*_{\lambda}+2\vep.
		$$
		Since $\vep$ is arbitrary, it follows $ \loss_\lambda(q_{\lambda,*}) = \loss^*_{\lambda}$, i.e., $q_{\lambda,*} \in \mathcal{A}_\lambda$. Next by using the fact that \eqref{eq:wass-conv-flow-est}, in particular, holds for $p/(p-q_0) \leq p'< p$, we have by \eqref{eq:drft-conv}
		\begin{align}\label{eq:b-bar-subseq-conv}
			\lf\|\bar b_{\hat \theta_{n_k,\lambda}} -\bar b_{q_{\lambda,*}}\ri\|_K \equiv \sup_{x \in K}\lf\|\bar b_{\hat \theta_{n_k,\lambda}}(x) -\bar b_{q_{\lambda,*}}(x)\ri\| \stackrel{n\rt \infty}\Rt 0,
		\end{align}
		where recall for notational simplicity, we denote $\bar b_{\hat \theta_{n,\lambda}}  \dfeq \bar b_{q^{(n)}_{\hat \theta_{n,\lambda}}}$ (see \eqref{eq:theta-est}).
		
		Now by the triangle inequality, 
		\begin{align*}
			\mathcal{W}_{p'}\lf(q^{(n)}_{\hat \theta_{n_k,\lambda}}, \mathcal{A}\ri) \leq&\  \mathcal{W}_{p'}\lf(q^{(n)}_{\hat \theta_{n_k,\lambda}}, q_{\lambda,*}\ri)+ \mathcal{W}_{p'}\lf(q_{\lambda,*}, \mathcal{A}\ri)\\
			\lf\|\bar b_{\hat \theta_{n_k,\lambda}}- \mathcal{M}_{\mrm{mle}}\ri\|_K \leq&\ \lf\|\bar b_{\hat \theta_{n_k,\lambda}} -\bar b_{q_{\lambda,*}}\ri\|_K + \lf\|\bar b_{q_{\lambda,*}}- \mathcal{M}_{\mrm{mle}}\ri\|_K.
		\end{align*}
		It follows from \eqref{eq:wass-conv-flow-est}, \eqref{eq:b-bar-subseq-conv} and the {\em claim} that
		$$\lim_{\lambda \rt 0} \lim_{k\rt \infty} \mathcal{W}_{p'}\lf(q^{(n)}_{\hat \theta_{n_k,\lambda}}, \mathcal{A}\ri) =0, \quad \lim_{\lambda\rt 0} \lim_{k\to\infty} \lf\|\bar b_{\hat \theta_{n_k,\lambda}} -\mathcal{M}_{\mrm{mle}}\ri\|_K =0.$$
		Since these limits are independent of choice of the specific subsequence $\{n_k\}$,  \eqref{eq:conv-rho-mle} and \eqref{eq:MLE-set-conv}  hold.
		
	\end{proof}
	\section{Neural network architecture}
	\label{sec:nn-architecture}
	
	We use normalizing flows at two levels. The first flow, \(f_\theta\), parameterizes the latent invariant distribution \(q_\theta=(f_\theta)_\#\rfdist\) appearing in the integral drift. The second flow, \(g_\vart\), parameterizes the variational distribution \(\rho_\vart=(g_\vart)_\#\rho_{\mathrm{ref}}\) over the parameters \(\theta\) of \(f_\theta\). Both maps are implemented using RealNVP affine coupling layers.
	
	\paragraph{RealNVP architecture.}
	Let \(h_\phi:\R^m\to\R^m\) denote a generic RealNVP flow with parameter \(\phi\). The map is written as a composition of \(K_{\mathrm{flow}}\) affine coupling layers,
	\begin{align*}
		y^0 = z, \qquad
		y^k = h_\phi^{(k)}(y^{k-1}), \qquad k=1,\ldots,K_{\mathrm{flow}},
	\end{align*}
	and \(h_\phi(z)=y^{K_{\mathrm{flow}}}\). At layer \(k\), the input is split as
	\[
	y^{k-1}=(y^{k-1}_a,y^{k-1}_b),
	\qquad
	m_a+m_b=m.
	\]
	The affine coupling transformation is
	\begin{align}
		\label{eq:realnvp-layer}
		y^k_a &= y^{k-1}_a, \nonumber\\
		y^k_b
		&=
		y^{k-1}_b
		\odot
		\exp\!\left(s_\phi^{(k)}(y^{k-1}_a)\right)
		+
		t_\phi^{(k)}(y^{k-1}_a),
	\end{align}
	where
	\[
	s_\phi^{(k)},t_\phi^{(k)}:\R^{m_a}\to\R^{m_b}
	\]
	are feed-forward neural networks, and \(\odot\) denotes componentwise multiplication. Alternating binary masks are used across layers so that all coordinates are transformed through the composition.
	
	The triangular structure of \eqref{eq:realnvp-layer} gives a tractable Jacobian determinant:
	\begin{align}
		\label{eq:realnvp-logdet}
		\log\left|
		\det \frac{\partial h_\phi(z)}{\partial z}
		\right|
		=
		\sum_{k=1}^{K_{\mathrm{flow}}}
		\sum_{i=1}^{m_b}
		s_\phi^{(k)}(y^{k-1}_a)_i .
	\end{align}
	Thus density evaluation and sampling are both efficient.

	\section{Training details}
	\label{subsec:training-details}
	
	All models are trained using the Adam optimizer with learning rate \(10^{-3}\). We do not use a learning-rate scheduler. To stabilize optimization, gradient norms are clipped at \(5.0\). Training is performed for \(100\) variational iterations. At each iteration, a minibatch of size \(B=500\) is sampled uniformly from the observed trajectory, and the minibatch log-likelihood is rescaled by \(M_0/B\) to approximate the full-data contribution to the ELBO.
	
	The ELBO is estimated using nested Monte Carlo sampling. We use \(K=100\) samples from the variational distribution \(\rho_\vart\) and \(L=100\) latent samples from \(\rfdist\) for each sampled parameter. To reduce memory usage, both Monte Carlo levels are evaluated in batchs: the parameter samples are processed with batch size \(20\), and the latent samples are processed with batch size \(25\).
	
	The latent transport map \(f_\theta\) is implemented as a RealNVP flow with \(2\) affine coupling layers and hidden dimension \(5\). The variational posterior flow \(g_\vart\) is implemented as a RealNVP flow with \(6\) affine coupling layers and hidden dimension \(256\). We use \(N=10\) latent particles, so that the latent dimension is
	\[
	D=N\cdot d .
	\]
	The kernel scale is fixed at \(\zeta=1.0\), the diffusion coefficient is fixed at \(\diffus=0.1\), and the Euler--Maruyama step size is \(\Delta=0.01\).
	
	For posterior evaluation, we use \(500\) samples from the variational distribution and \(1000\) latent samples per parameter sample. These computations are also performed using batching, with parameter batch size \(20\) and latent batch size \(100\).
	\section{Further experiments}
	\paragraph{Finite-time law comparison for non-ergodic dynamics: } We compare the finite-time marginal laws of the true averaged SDE and the learned averaged SDE.
	
	For each experiment, we generate \(L\) independent trajectories from both systems using the same Brownian noise realization. At every time point \(t_k\), this produces two empirical distributions consisting of \(L\) samples each. We then compare these distributions using the Kolmogorov--Smirnov statistic and the Wasserstein distance.
	
	Figures \ref{fig:law_t1_t15} show kernel density estimate (KDE) comparisons of the empirical laws at selected time instances for increasing numbers of sample paths \(L\). As \(L\) increases, the empirical distributions become more stable and the agreement between the learned averaged dynamics and the true averaged dynamics becomes more apparent.
	
	The comparisons demonstrate that the inferred drift reproduces not only individual sample-path behavior, but also the finite-time probabilistic structure of the averaged system.
	\begin{figure}[H]
		\centering
		
		\begin{minipage}{0.24\textwidth}
			\centering
			\includegraphics[width=\linewidth]{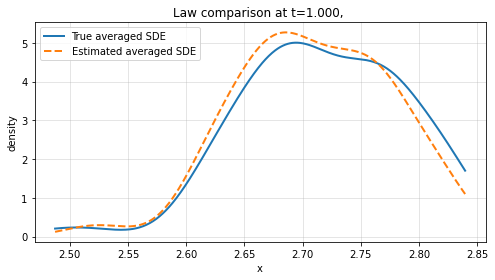}
			{\scriptsize $L=100$}
		\end{minipage}
		\hfill
		\begin{minipage}{0.24\textwidth}
			\centering
			\includegraphics[width=\linewidth]{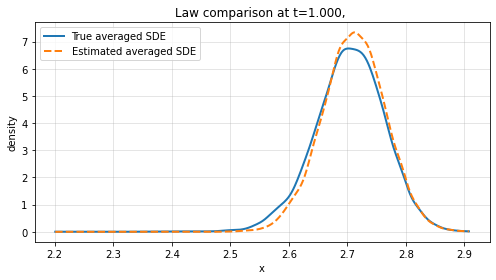}
			{\scriptsize $L=1000$}
		\end{minipage}
		\hfill
		\begin{minipage}{0.24\textwidth}
			\centering
			\includegraphics[width=\linewidth]{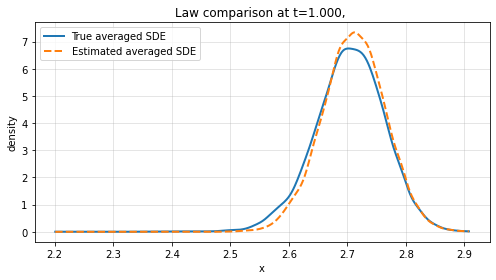}
			{\scriptsize $L=10000$}
		\end{minipage}
		\hfill
		\begin{minipage}{0.24\textwidth}
			\centering
			\includegraphics[width=\linewidth]{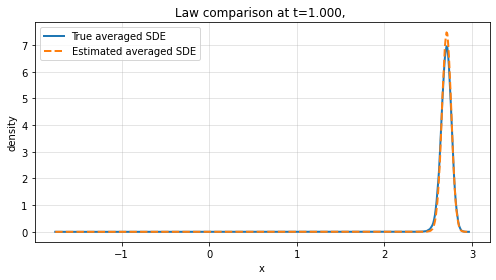}
			{\scriptsize $L=100000$}
		\end{minipage}
		
		\vspace{0.4cm}
		
		\begin{minipage}{0.24\textwidth}
			\centering
			\includegraphics[width=\linewidth]{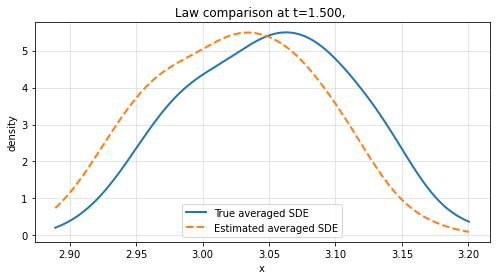}
			{\scriptsize $L=100$}
		\end{minipage}
		\hfill
		\begin{minipage}{0.24\textwidth}
			\centering
			\includegraphics[width=\linewidth]{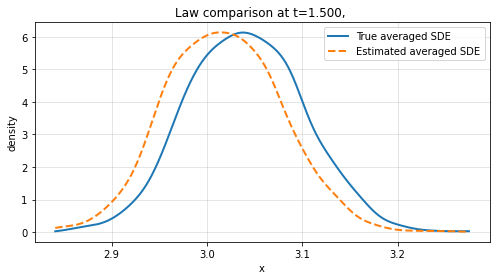}
			{\scriptsize $L=1000$}
		\end{minipage}
		\hfill
		\begin{minipage}{0.24\textwidth}
			\centering
			\includegraphics[width=\linewidth]{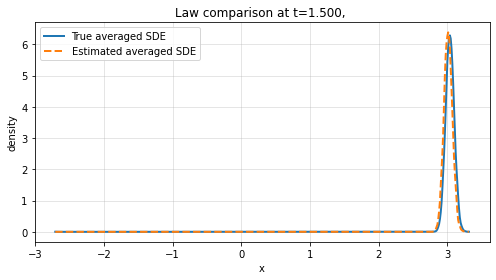}
			{\scriptsize $L=10000$}
		\end{minipage}
		\hfill
		\begin{minipage}{0.24\textwidth}
			\centering
			\includegraphics[width=\linewidth]{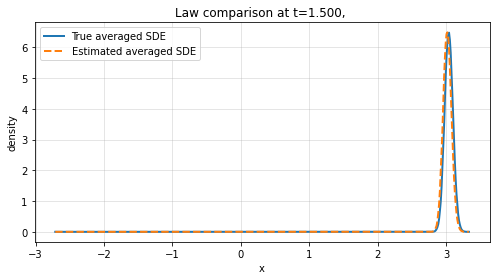}
			{\scriptsize $L=100000$}
		\end{minipage}
		
		\caption{Law comparison at $t=1$ and $t=1.5$ for different numbers of sample paths $L$.}
		\label{fig:law_t1_t15}
	\end{figure}
	\paragraph{Model :Variant of Double-well dynamics.}
	\paragraph{Double-well example:}We consider a slow--fast SDE system where the fast dynamics evolve with the invariant distribution given by a four-dimensional von Mises law. The slow component satisfies a nonlinear double-well type dynamics with drift{\setlength{\abovedisplayskip}{2pt}\setlength{\belowdisplayskip}{2pt}\begin{align*}b(x,y_1,y_2,y_3,y_4)=\frac{x-x^3}{1+x^2+y_1^2+\sin\!(1+y_2^2)+\log(1+y_3^2)+y_4^4}.\end{align*}}
	
	The cubic structure in the numerator induces a double-well behavior in the effective averaged dynamics, while the denominator introduces a highly nonlinear dependence on the fast variables.
	
	\paragraph{Experimental Design:}Synthetic observations are generated by simulating the corresponding multiscale slow--fast SDE system using Euler--Maruyama discretization over a fixed time horizon. The fast dynamics are simulated independently and subsequently coupled with the slow process through the above interaction drift. The averaged drift is then learned using the proposed variational flow framework.
	
	\paragraph{Comparison:}Figure~\ref{fig:nips_app_comparison} compares the true averaged drift with the estimated averaged drift obtained from the variational normalizing flow model. The learned drift accurately captures both the global double-well structure and the local nonlinear behavior near the origin, demonstrating that the proposed method is capable of recovering complicated averaged dynamics induced by nonlinear fast variables.
	The attached simulation code can be cited separately if needed. 
	\begin{figure}[t]
		\centering
		\includegraphics[width=0.35\textwidth]{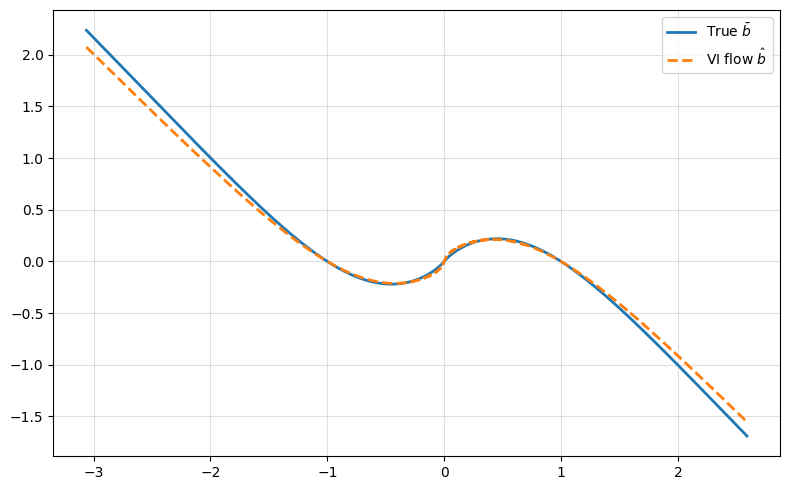}
		\caption{Comparison of the estimated drift and the averaged drift.}
		\label{fig:nips_app_comparison}
	\end{figure}
	\newpage

\end{document}